\documentclass[conference]{milcom_style/IEEEtran}
\IEEEoverridecommandlockouts
\usepackage{cite}
\usepackage{amsmath,amssymb,amsfonts,amsthm}
\usepackage{algorithmic}
\usepackage{graphicx}
\usepackage{textcomp}
\usepackage{xcolor}
\usepackage{booktabs}
\usepackage[table]{xcolor}
\usepackage{multirow}
\usepackage{subcaption}
\usepackage{url}
\def\BibTeX{{\rm B\kern-.05em{\sc i\kern-.025em b}\kern-.08em
    T\kern-.1667em\lower.7ex\hbox{E}\kern-.125emX}}

\newtheorem{theorem}{Theorem}[section]

\usepackage{mathtools}
\usepackage{dsfont}
\usepackage[ruled,linesnumbered]{algorithm2e}
\SetKwComment{Comment}{/* }{ */}
\newlength\imagewidth
\newlength\imagescale

\newenvironment{nospaceflalign*}
 {\setlength{\abovedisplayskip}{0pt}\setlength{\belowdisplayskip}{0pt}%
  \csname flalign*\endcsname}
 {\csname endflalign*\endcsname\ignorespacesafterend}

\DeclareMathOperator*{\argmax}{arg\,max}

\usepackage{comment}

\begin{document}

\title{Robust and Explainable Divide-and-Conquer Learning for Intrusion Detection} 

\author{%
    \IEEEauthorblockN{%
        Yan Zhou\IEEEauthorrefmark{1}, 
        Kevin Hamlen\IEEEauthorrefmark{1}, 
        Michael De Lucia \IEEEauthorrefmark{2}, 
        Murat Kantarcioglu \IEEEauthorrefmark{3}, \\
        Latifur Khan \IEEEauthorrefmark{1},  
        Sharad Mehrotra \IEEEauthorrefmark{4}, 
        Ananthram Swami \IEEEauthorrefmark{2}, 
        Bhavani Thuraisingham\IEEEauthorrefmark{1}%
        }%
    \thanks{\IEEEauthorrefmark{1} University of Texas at Dallas, Richardson, TX, \{yan.zhou2, hamlen, lkhan, bxt043000\}@utdallas.edu}%
    \thanks{\IEEEauthorrefmark{2} DEVCOM, Army Research Laboratory, Adelphi, MD \{michael.j.delucia2.civ,ananthram.swami.civ\}@army.mil}%
    \thanks{\IEEEauthorrefmark{3} Department of Computer Science, Virginia Tech, Blacksburg, VA, muratk@vt.edu}%
    \thanks{\IEEEauthorrefmark{4} University of California, Irvine, sharad@ics.uci.edu}%
}%


\maketitle
\begin{abstract}
Machine learning-based intrusion detection requires complex models to capture patterns in high-dimensional, noisy, and class-imbalanced raw network traffic, yet deploying such models remains impractical on resource-constrained devices with limited processing power and memory. In this paper, we present a correlation-aware divide-and-conquer learning technique that decomposes a complex learning problem into smaller, more manageable subproblems. This enables lightweight models as simple as decision trees to be trained on focused subtasks, yielding up to 43.3\% higher local accuracy and up to 257  times reduction in model size on real-world network intrusion detection datasets, while also improving adversarial robustness and explainability. 
\end{abstract}
\begin{IEEEkeywords}
Intrusion detection systems, subproblem correlation-aware learning, resource-constrained learning, robustness, cost, explainability
\end{IEEEkeywords}

\section{Introduction}
\label{sec:intro}

Machine learning is widely used in network intrusion detection. The high-dimensional, noisy, and imbalanced nature of raw network traffic demands complex models for accurate detection. Such models are computationally intensive and impractical for resource-constrained  devices. Moreover, complex models are prone to overfitting, vulnerable to adversarial attacks~\cite{onmodcomplexity,ma2021understanding}, and hard for human experts to interpret~\cite{electronics11132082}. {\em This raises a natural question: can complex learning tasks be decomposed into simpler subtasks for which lightweight, effective, and robust models can be trained independently?}

Divide-and-conquer techniques have been developed to address large, complex learning problems by breaking them into smaller subproblems that can be solved independently~\cite{1031186,pmlr-v32-hsieha14, cascadesvm, pmlr-v202-ghosh23c}. However, existing approaches typically optimize for a single objective such as accuracy or efficiency, while real-world applications also demand robustness and explainability where adversarial attacks are expected and human oversight is essential~\cite{10.3233JCS-210094,electronics11132082}. In this paper, we present SCAL (Subproblem Correlation-Aware Learning), a novel divide-and-conquer technique that jointly optimizes accuracy, cost, robustness, and explainability, along with a holistic comparison between a single model trained globally and smaller models trained locally in a divide-and-conquer fashion.

SCAL automatically partitions a learning problem into smaller subproblems by grouping correlated traffic types based on network traffic fingerprints. Unlike existing techniques~\cite{Jacobs1991Adaptive, NIPS1991_59b90e10}, SCAL falls back to a single learning problem  when decomposition offers no benefit. SCAL is organized as a three-tier hierarchy: a subproblem generator at the root groups highly correlated classes into independent subproblems, for example, VPN and non-VPN SSH traffic are grouped together, separate from unencrypted streaming traffic. An instance distributor in the middle routes incoming traffic to the appropriate subproblem, and independent local classifiers at the leaves are trained to solve subproblems in parallel. This design yields local models that are more accurate, compact, robust, and interpretable than a single global model.

The main contributions of this work are:
\begin{itemize}
\itemsep=0pt
    \item a divide-and-conquer learning technique that automatically decomposes a complex learning problem into lower-complexity subproblems whenever beneficial;
    \item a holistic evaluation comparing global and local models across four objectives: accuracy, cost, robustness, and explainability;
    \item an extensive case study on payload-based network intrusion detection, a domain characterized by severe class imbalance and strict resource constraints.
\end{itemize}

The rest of this paper is organized as follows. Section~\ref{sec:related} reviews related work. Section~\ref{sec:method} describes the SCAL technique in detail. Section~\ref{sec:exp} evaluates SCAL in terms of accuracy, cost, robustness, and explainability on network payload datasets. Finally, Section~\ref{sec:conclude} concludes our work and discusses limitations.

\section{Related Work}
\label{sec:related} 

The divide-and-conquer principle underlies many modular learning approaches, including Mixture of Experts~\cite{Jacobs1991Adaptive,NIPS1991_59b90e10,716791}, ensemble methods~\cite{boosting1996,breiman96}, and Random Forests~\cite{breiman2001random}, all of which partition the input space across specialized models that are combined into a final hypothesis~\cite{1368565}. However, these techniques focus primarily on improving overall accuracy and do not explicitly account for subproblem characteristics such as local concepts, discriminatory capacity, or class correlations.

Modular learning has also been explored by partitioning the output space. Binary Hierarchical Classifiers(BHC)~\cite{kumar2002} recursively group classes into two subgroups, decomposing a $C$-class problem into $C-1$ binary subproblems using techniques such as simulated annealing~\cite{1999SPIE3722K} or max-cut~\cite{1368565}. While effective for large output spaces, these methods indiscriminately impose a binary hierarchy on all problems, regardless of whether such a structure naturally fits the data.

Hierarchical Classification~\cite {hcsurvey2011}, widely used in text mining~\cite{1031186,1390720}, organizes classification problems into a tree of meta-classes, either defined by a pre-established taxonomy~\cite{657461,2835898} or automatically constructed by grouping similar classes~\cite{1062843}. Like other modular learning techniques, its primary goal is to improve classification accuracy.
 
Divide-and-conquer techniques have also targeted efficiency and explainability. DC-SVM~\cite{pmlr-v32-hsieha14} partitions kernel SVM problems into smaller subproblems via kernel clustering, achieving much faster training than standard SVMs. Mixture of Interpretable Experts (MoIE)~\cite{pmlr-v202-ghosh23c} routes subsets of samples through interpretable models to explain local concepts.

Unlike existing work that optimizes for a single objective, we comprehensively evaluate our SCAL technique across four dimensions: \textit{accuracy, cost, robustness, and explainability} in the context of network intrusion detection under adversarial conditions and resource constraints.

\section{Methodology}
\label{sec:method}

We first establish the theoretical conditions for divide-and-conquer learning, then present SCAL, which automatically decomposes a learning problem into subproblems by grouping highly correlated classes.  As an example, in network intrusion detection, DoS attack variants such as Hulk, GoldenEye, and Slowloris all exhaust server resources and are highly correlated, making them natural candidates for grouping into a single subproblem. While designed for network payload data, the approach generalizes to other domains.

\subsection{Subtask Correlation and Accuracy}
\label{sec:subtask}
For simplicity, we consider a learning problem with  two learning subtasks. Let $X \in \mathbb{R}^d$ and $Y \in \{0,1,2\}$. We compare the global learning task $f: X \rightarrow Y\in \{0,1,2\}$ with the composite $\hat{f}$ of two subtasks $f_1: X \rightarrow Y_1 \in \{0,1\}$ and $f_2: X \rightarrow Y_2 \in \{0,2\}$. In essence, the composite $\hat{f}:X \rightarrow Y$ is:
\[
\hat{f}(X) = \begin{cases}
    i, & \text{if } f_i(X) > 0 \land f_j(X) = 0 \\
    0, & \text{if } f_i(X) = f_j(X) = 0 \\
    \text{ND}, & \text{if } f_i(X) \cdot f_j(X) > 0
\end{cases}
\]
where $i,j\in\{1,2\}$ and ND denotes ``not defined''. Note that the problem under consideration can be easily extended to multiple learning subtasks.  
\begin{theorem}
\label{the:correlation}
If two learning subtasks $f_1$ and $f_2$ are positively correlated, learning a single task  $f$ yields lower error than learning $f_1$ and $f_2$ separately to construct the composite classifier $\hat{f}$. 
\end{theorem}
\begin{proof}
The correlation coefficient between the two learning subtasks $f_1$ and $f_2$ is:
\begin{align}
\rho_{1,2} & = \frac{E(\mathds{1}_1\mathds{1}_2)-E(\mathds{1}_1)E(\mathds{1}_2)}{\sqrt{Var(\mathds{1}_1)Var(\mathds{1}_2)}}
\end{align}
where $\mathds{1}_{i \in \{1,2\}}$ is the indicator random variable of the event $f_i(X)=i$, and  $\mathds{1}_1\mathds{1}_2$ is the indicator variable of the event in the conflict region $\mathcal{C} := \{ X \in \mathbb{R}^d \mid f_1(X) = 1  \land  f_2(X) = 2 \}$ in which $\hat{f}(X)$ predicting both 1 and 2 is ambiguous. Thus,  
\begin{align}
E(\mathds{1}_1\mathds{1}_2)& = p = Pr((f_1(X) = 1) \land (f_2(X) = 2))\\
E(\mathds{1}_1) & = p_1 = Pr(f_1(X) = 1) \\
E(\mathds{1}_2) & = p_2 = Pr(f_2(X) = 2) 
\end{align}
Therefore,
\[
\rho_{1,2} = \frac{p-p_1p_2}{\sqrt{p_1(1-p_1)p_2(1-p_2)}}
\]
When two subtasks are highly positively correlated, that is, $\rho_{1,2} > 0$, the probability of the conflict region $Pr(X\in\mathcal{C})$ increases (with respect to them being independent) as $p > p_1p_2$: 
\begin{align*}
Pr(X\in\mathcal{C}) & = Pr((f_1(X) = i) \land (f_2(X) = j))    \\
 & = Pr(f_i(X) = i| f_j(X) = j) \cdot Pr(f_j(X) = j )    \\
&  > Pr(f_i(X) = i) \cdot Pr (f_j(X) = j) 
\end{align*}
where $i \neq j$ and $i, j \in \{1,2\}$. 

Given $(X,Y)$, $\hat{f}(X) \neq Y$ when $\hat{f}$ misclassifies $X$.  The excess risk incurred by learning subtasks $f_1$ and $f_2$ separately over the conflict region $\mathcal{C}$ and the non-conflict region $\mathcal{C}' := \{ X \in \mathbb{R}^d \mid f_1(X) \neq  f_2(X) \}$, relative to learning the joint classifier $f$, is:
\begin{align*}
R (\hat{f}, f) & = \mathbb{E}[\mathds{1}_{ \hat{f}(X) \ne Y }] - \mathbb{E}[\mathds{1}_{{f}(X) \ne Y }] \\
& = \mathbb{E}[\mathds{1}_{ \hat{f}(X) \ne Y|X\in \mathcal{C}' }] + \mathbb{E}[\mathds{1}_{ \hat{f}(X) \ne Y|X\in \mathcal{C} }] \\
& \;\;\;\;\;- \mathbb{E}[\mathds{1}_{{f}(X) \ne Y }] \\
& \approx Pr(\hat{f}(X) \ne Y \mid X \in \mathcal{C}) \cdot  Pr(X \in \mathcal{C})
\end{align*}
as $f$ does not suffer from the ambiguity in $\mathcal{C}$, unlike $\hat{f}$. Note that $\mathbb{E}[\mathds{1}_{ \hat{f}(X) \ne Y|X\in \mathcal{C}' }] \approx \mathbb{E}[\mathds{1}_{{f}(X) \ne Y }]$ is based on the assumption that both $f$ and $\hat{f}$ approach Bayes classifier with sufficient training data~\cite{ishida2023is}. The difference in classification error is attributed to the conflict region. Thus, learning the composite $\hat{f}$ of $f_1$ and $f_2$ is suboptimal when $Pr(X\in\mathcal{C})$ is non-trivial.
\end{proof}

\subsection{Subproblem Decomposition}
We seek a partition into $k$ subproblems that minimizes intra-subproblem learning loss while maximizing inter-subproblem dissimilarity:

\begin{align*}
   & \min_{\{X_1, \ldots, X_k\} \subset X} \sum_{X_i\subset X}Pr(f_i(X_i)\neq y_i|X_i)\cdot Pr(X_i)  \\
    s.t. \;\;\; & 
    X_i \in \{\argmax_{X_{i \in [1,k]} \subset X} D(X_1, \ldots, X_k)\} \\
   & X = \bigcup_{i=1}^{k} X_i, X_i \cap X_j = \emptyset\; for\;i \neq j
\end{align*}

where $f_i$ is trained on subproblem $(X_i, y_i)$, and $D(X_1, \ldots, X_k)$ measures inter-subproblem dissimilarity via pairwise class correlations across subproblems. The number of subproblems $k$ can be fixed or determined via hyperparameter search during offline training.

To compute class correlations, we assign each instance a fingerprint---a vector of code lengths obtained by compressing the payload under each class-specific compression model: $\tilde{x} = [\tilde{x}_1, \ldots, \tilde{x}_{|C|}]$. Figure~\ref{fig:fingerprint} illustrates this process.

\begin{figure}[!htb]
    \centering
    \includegraphics[clip,trim=0cm 0cm 0cm 0cm, width=0.4\textwidth]{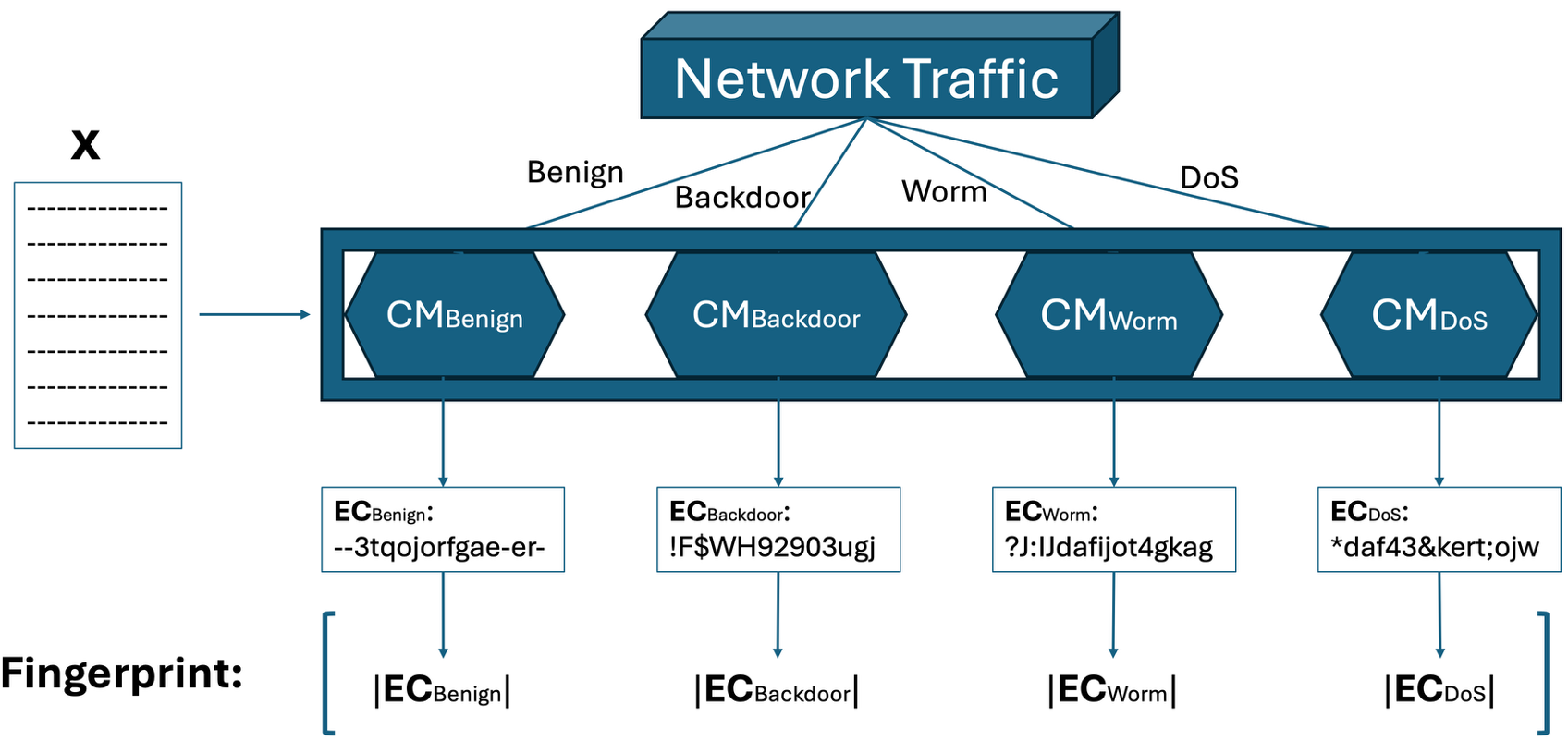}
    \caption{\label{fig:fingerprint} Fingerprint calculation for a given payload $x$.}
\end{figure}

From the fingerprint profile $\widetilde{X}$ of the training data, we compute a $|C|\times |C|$ correlation matrix $R$, where each entry captures the correlation between two classes. Highly correlated classes are grouped into the same subproblem to facilitate learning of fine-grained local distinctions. 

\subsection{Subproblem Correlation-Aware Learning: SCAL}

Given $|C|$ classes, SCAL first builds $|C|$ compression models to compute instance fingerprints, then derives a class  correlation matrix (Lines 1--5 in Algorithm~\ref{alg:cascade}). Hierarchical clustering on this matrix generates the subproblem partition (Line 6). An instance distributor is trained to route new instances to the appropriate subproblem (Line 7), and a local classifier is trained independently for each subproblem (Lines 8--11). At inference, each instance is first routed by the distributor, then classified by the corresponding local classifier. Instances routed to a single-class subproblem are classified directly without a local classifier.

\RestyleAlgo{ruled}
\begin{algorithm}
\caption{\label{alg:cascade} SCAL}
\DontPrintSemicolon
\KwData{X, y \Comment*[r]{input data and label}}
\KwResult{$C_s$, $C_{[1,k]}$ \Comment*[r]{local classifiers}}

\For{$i \gets 1$ \KwTo $|C|$ }{
$CM_{i} \gets compress(X_i,y_i)$ \Comment*[r]{Compression model for class $i$}
 }  
$\widetilde{X} \gets generate\_fingerprints(CM_{[1,\ldots,|C|]}, X)$\\
$Corr \gets R(\widetilde{X})$ \Comment*[r]{corr. coefficient}
$SubProb_{[1,\ldots,k]} \gets generate\_subproblem(Corr,k)$\\
$C_s \gets train\_instance\_distributor(X, SubProb)$\\
\For{$i \gets 1$ \KwTo $k$ }{
$X_i, y_i \gets make\_subproblem\_instances(C_s, X, y)$ \\
$C_{i} \gets train\_subclassifier(X_i,y_i)$
 }  
\end{algorithm}

\section{Experimental Evaluation}
\label{sec:exp}
We compare a single global classifier against local classifiers across four dimensions: accuracy, cost, robustness, and explainability. Results show that local classifiers achieve higher local accuracy, lower model complexity, stronger adversarial resilience, and greater interpretability. All experiments were run on an Intel Xeon W-2135 3.70GHz CPU with 250GB RAM, repeated 10 times with averaged results and standard deviations reported.

\subsection{Datasets and Models}
We evaluate on four real-world network intrusion detection datasets, focusing on the more challenging payload-based version of each. Payloads were extracted and represented as 1500-byte sequences, zero-padded as needed, and combined with four packet features: time to live, total length, protocol, and time duration. Source/destination IP and port information was excluded to prevent models from learning spurious associations.

\begin{itemize}
\itemsep=0pt
\item \textbf{UNSW-NB15}~\cite{unsw}: 79,881 instances (59,910 train / 19,971 test), 10 classes. Four of nine malicious classes each account for less than 2\% of traffic.
\item \textbf{CIC-IDS2017}~\cite{DBLP:conf/icissp/SharafaldinLG18}: 1,410,255 instances (1,057,691 train / 352,564 test), 15 classes. Six malicious classes each represent at most 1\% of traffic.
\item \textbf{IoT}~\cite{qacj-3x32-23}: 631,486 instances (473,614 train / 157,872 test), 10 classes. Over 95\% of traffic is benign, making this the most imbalanced dataset.
\item \textbf{VNAT}~\cite{vnat}: 179,996 instances (159,996 train / 20,000 test), 20 classes of encrypted and unencrypted application traffic, each containing exactly 8,000 instances.
\end{itemize}

We compare four global classifiers XGBoost, Random Forest, Logistic Regression, and Deep Neural Network against three divide-and-conquer techniques: Mixture of Experts~\cite{saves2024smt}, Binary Hierarchical Classifier (BHC)~\cite{kumar2002}, and SCAL. Local classifiers in SCAL and BHC are decision trees, chosen for their simplicity and interpretability. As an ablation, we also evaluate pseudo-SCAL, a variant of SCAL with randomly generated subproblems, to assess the impact of correlation-aware partitioning. Level-2 {\em zstd} compression algorithm~\footnote{\url{https://github.com/facebook/zstd}} was chosen to build class-specific compression models for its lossless compression, favorable compression ratios, and real-time performance. 

\subsection{Accuracy}
Our first experiment investigates whether divide-and-conquer techniques improve accuracy on local concepts using simple models such as decision trees, particularly important in imbalanced datasets where majority-class prediction can yield misleadingly high overall accuracy (e.g., predicting Benign for all instances achieves over 95\% accuracy on the IoT dataset). Model calibration~\cite{Kuhn_13} was applied to align predicted class probabilities with the true data distribution, using cross-validation with the number of folds constrained by the smallest class size. F1 scores are reported due to class imbalance.

\begin{table*}[!htb]
\caption{\label{tab:f1} F1 scores of LR, DNN, RF, XGB, MoE, HC, Pseudo-SCAL, and SCAL on the four datasets.}
\scriptsize
\begin{tabular}{|
>{\columncolor[HTML]{D4D4D4}}c |c |c |c |c |c|c|c|c|c|c|c|c|c|c|}
\hline
\cellcolor[HTML]{B0B3B2} & \cellcolor[HTML]{B0B3B2}\textbf{LR} & \cellcolor[HTML]{B0B3B2}\textbf{DNN} & \cellcolor[HTML]{B0B3B2}\textbf{RF} & \cellcolor[HTML]{B0B3B2}\textbf{XGB} & \cellcolor[HTML]{B0B3B2}\textbf{MoE} & \cellcolor[HTML]{B0B3B2}\textbf{HC} & \cellcolor[HTML]{B0B3B2}\textbf{Pseudo-SCAL} & \cellcolor[HTML]{B0B3B2}\textbf{SCAL}\\ \hline
\textbf{UNSW}            & $0.429 \pm 0.002$ & $0.363 \pm 0.042$ & $0.693 \pm 0.011$            & $0.689 \pm 0.004$      & $0.373 \pm 0.023$           & $0.668 \pm 0.004$ & $0.669 \pm 0.020$ & $\mathbf{0.713 \pm 0.008}$\\ \hline
\textbf{CIC}             & $0.606 \pm 0.005$ & $0.520 \pm 0.015$ & $0.701 \pm 0.030$            & $0.691 \pm 0.001$      & $0.280 \pm 0.036$           & $0.707 \pm 0.006$ & $0.725 \pm 0.026$ & $\mathbf{0.735 \pm 0.003}$\\ \hline
\textbf{IoT}             & $0.462 \pm 0.010$ & $0.424 \pm 0.045$ & $0.859 \pm 0.007$            & $0.778 \pm 0.005$      & $0.387 \pm 0.002$           & $0.747 \pm 0.012$ & $0.849 \pm 0.006$ & $\mathbf{0.860 \pm 0.004}$\\ \hline
\textbf{VNAT}            & $0.420 \pm 0.006$ & $0.682 \pm 0.031$ & $0.730 \pm 0.003$            & $\mathbf{0.774 \pm 0.004}$     & $0.179 \pm 0.022$            & $0.743 \pm 0.001$ & $0.733 \pm 0.011$ & $0.770 \pm 0.001$\\ \hline
\end{tabular}
\end{table*}

Table~\ref{tab:f1} reports F1 scores for all classifiers across the four datasets. Key findings are: (a.) No subproblems are generated for IoT, confirming the ``no harm'' principle of SCAL. (b.) SCAL achieved the best results on the three imbalanced datasets. XGBoost was marginally better on the balanced VNAT dataset but performed significantly worse on imbalanced ones. (c.) Random forest was the best global classifier overall, while pseudo-SCAL ranked second among divide-and-conquer methods. (d.) SCAL, random forest, BHC, and XGBoost all substantially outperformed mixture of experts, logistic regression, and DNN.

It is worth noting that on UNSW-NB15, what appears to be a modest 2\% gain in overall F1 by SCAL  translates to substantial local improvements of 39.8\% for \textit{DoS} and 43.3\% for \textit{analysis}, with no degradation in \textit{Benign} accuracy. This disproportionate local improvement was observed across the other datasets. Also note that pseudo-SCAL with random partitions underperformed SCAL and in some cases even the global models, confirming that correlation-aware partitioning is essential for accuracy gains.

\subsection{Cost}
\label{sec:cost}
Our second experiment compares the cost of a global random forest classifier against local decision trees in SCAL, as shown in Table~\ref{tab:cost}. Random forest was selected as the best global classifier. SCAL significantly reduced both model size and training time across all datasets, except on VNAT where both models spent considerable time attempting to fit the encrypted traffic without success.

\begin{table}[ht]
\centering
\caption{Model size (MB) and training time (sec) of Random Forest (RF) and SCAL.}
\label{tab:cost}
\begin{tabular}{llrr}
\toprule
\textbf{Dataset} & \textbf{Model} & \textbf{Size (MB)} & \textbf{Time (sec)} \\
\midrule
\multirow{2}{*}{UNSW} 
  & RF & $85.623 \pm 0.266$  & $5.038 \pm 0.456$ \\
  & SCAL          & $2.539 \pm 0.009$   & $1.568 \pm 0.159$ \\
\midrule
\multirow{2}{*}{CIC}
  & RF & $5258.433 \pm 478.481$ & $1159.680 \pm 151.779$ \\
  & SCAL          & $125.006 \pm 12.842$   & $839.885 \pm 112.744$ \\
\midrule
\multirow{2}{*}{IoT}
  & RF  & $46.554 \pm 4.305$  & $184.097 \pm 19.653$ \\
  & SCAL          & $0.181 \pm 0.012$   & $80.420 \pm 7.949$ \\
\midrule
\multirow{2}{*}{VNAT}
  & RF & $1236.456 \pm 142.989$ & $115.884 \pm 13.513$ \\
  & SCAL          & $15.048 \pm 1.670$     & $112.309 \pm 12.565$ \\
\bottomrule
\end{tabular}
\end{table}

\begin{figure}[!htb]
    \centering
    \subcaptionbox{\label{fig:unsw-comp} UNSW F1/complexity}{
        \includegraphics[clip,trim=0.2cm 0cm 0.2cm 0cm, width=0.2325\textwidth]{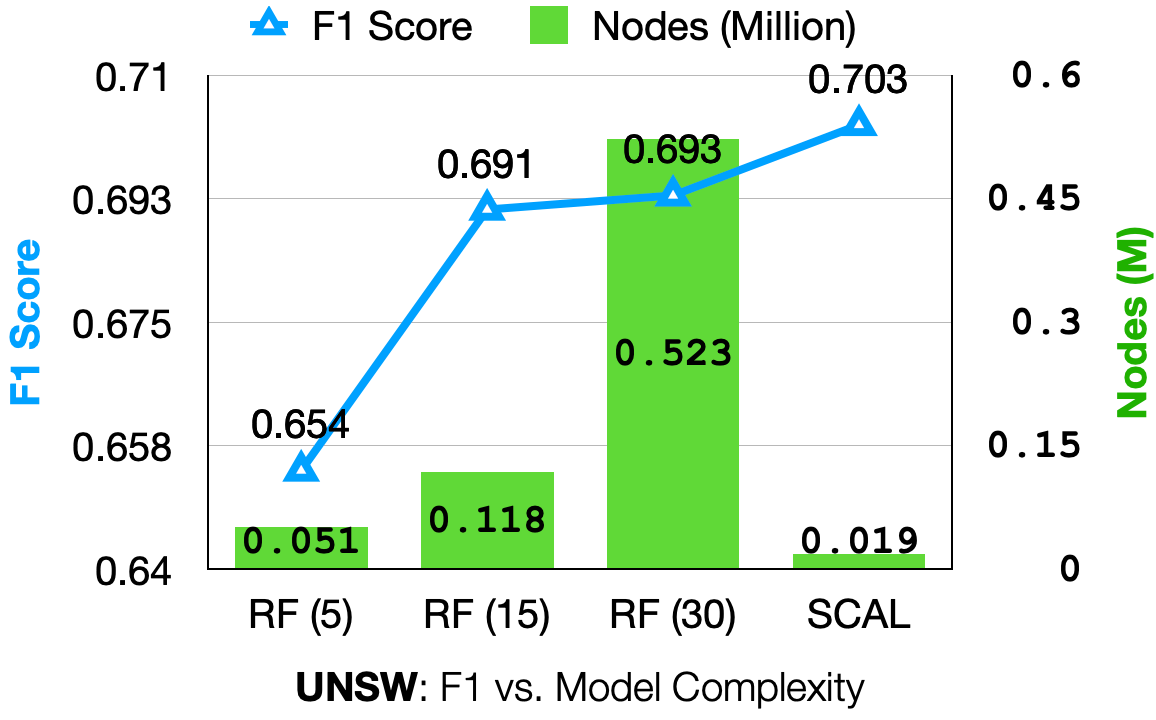}}
    \subcaptionbox{\label{fig:cic-comp} CIC F1/Complexity}{
        \includegraphics[clip,trim=0.2cm 0cm 0.2cm 0cm, width=0.2325\textwidth]{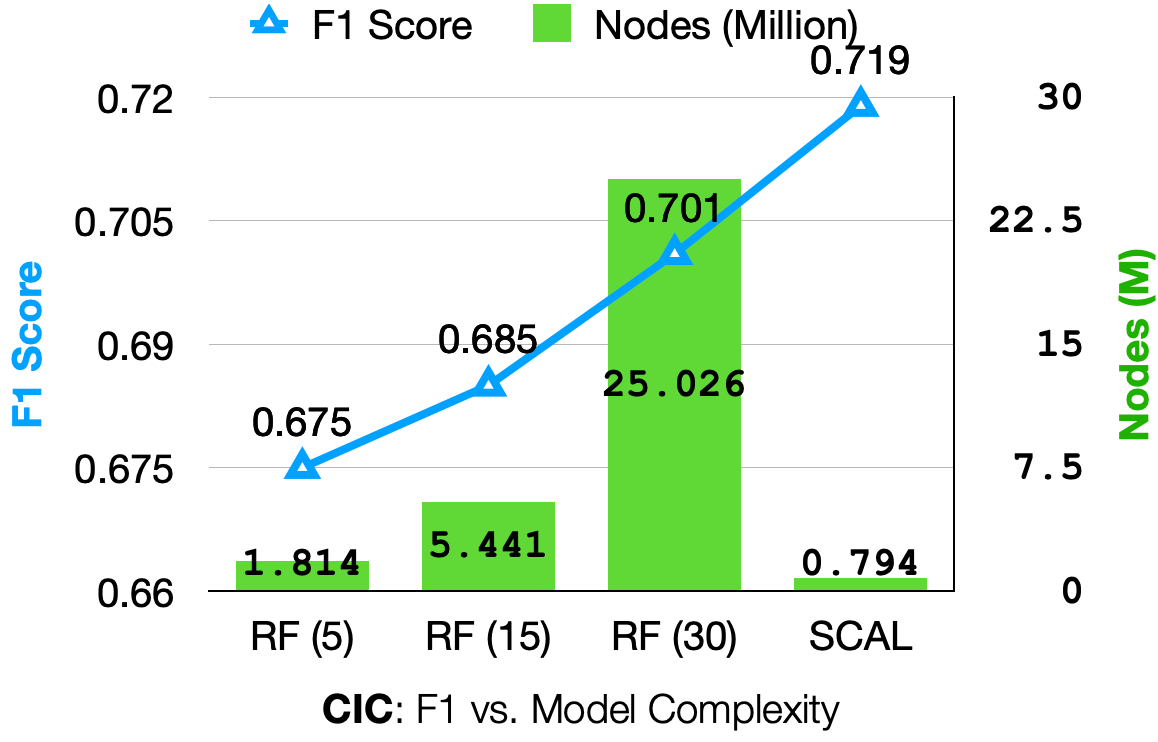}}
    \subcaptionbox{\label{fig:iot-comp} IoT F1/Complexity}{
        \includegraphics[clip,trim=0.2cm 0cm 0.2cm 0cm, width=0.2325\textwidth]{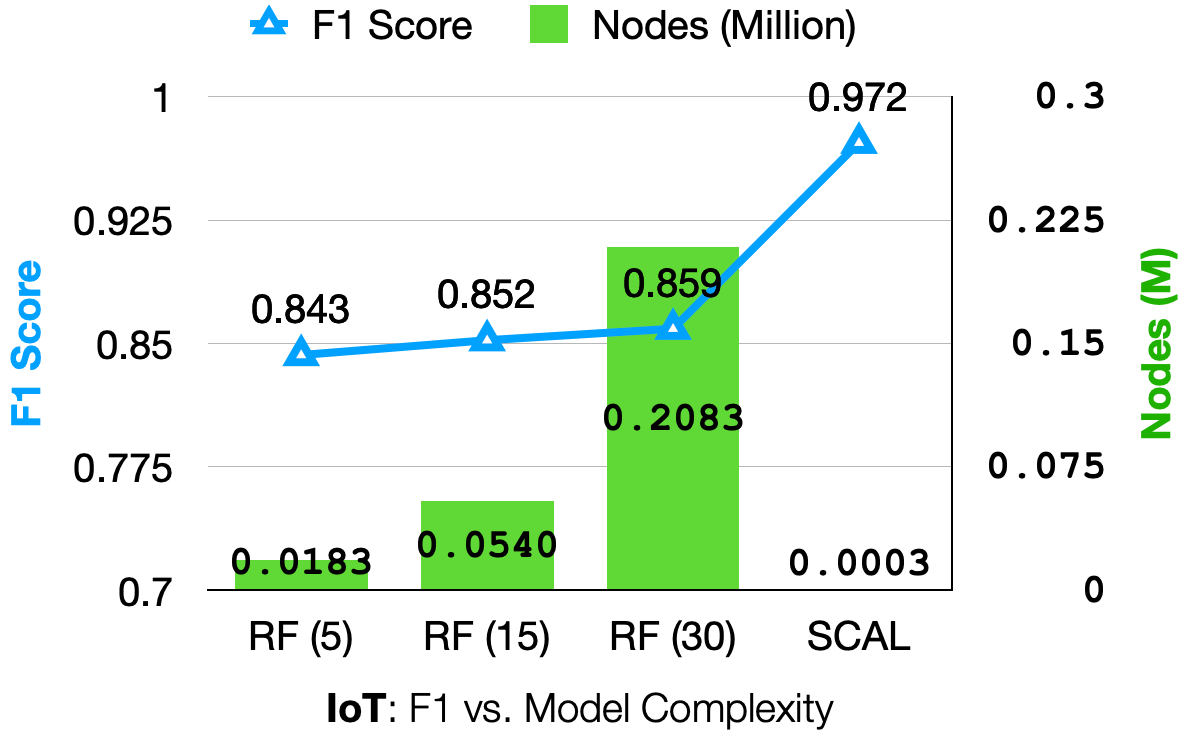}}
    \subcaptionbox{\label{fig:iot-comp} VNAT F1/Complexity}{
        \includegraphics[clip,trim=0.2cm 0cm 0.2cm 0cm, width=0.2325\textwidth]{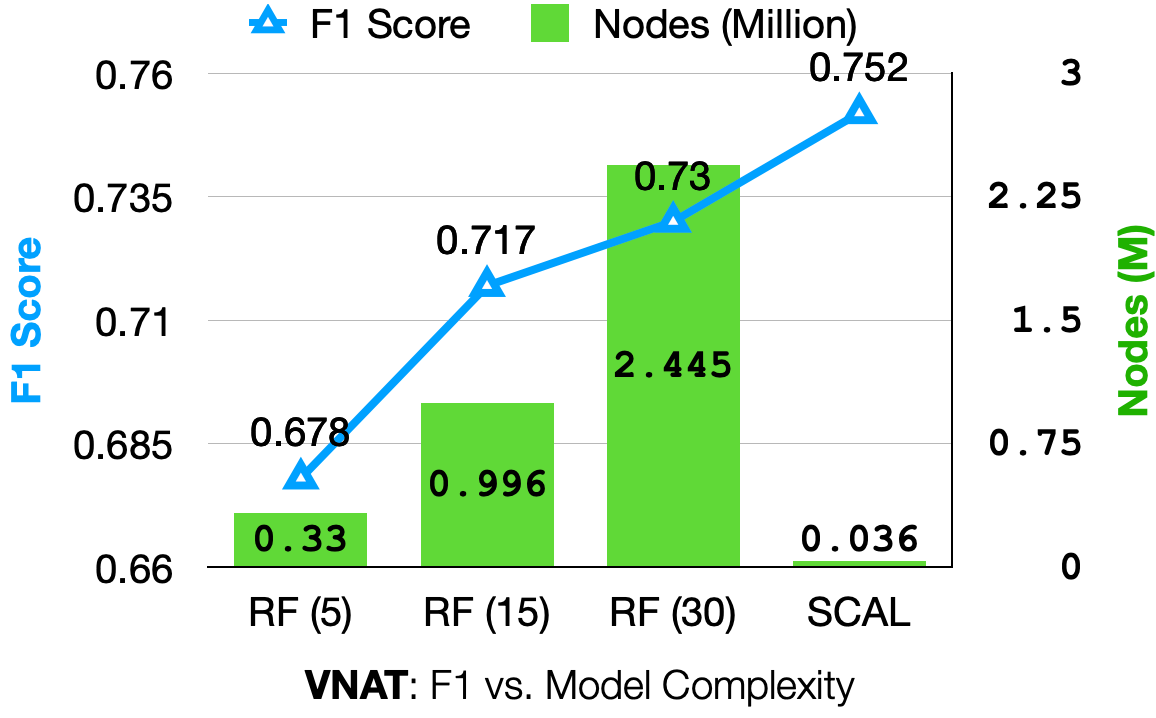}}
    \caption{\label{fig:comp} The impact of the number of estimators in the global random forest classifier on performance and model complexity.}
\end{figure}

\begin{figure*}[!htb]
    \centering
    \includegraphics[clip,trim=0cm 0cm 0cm 0cm, width=0.75\textwidth]{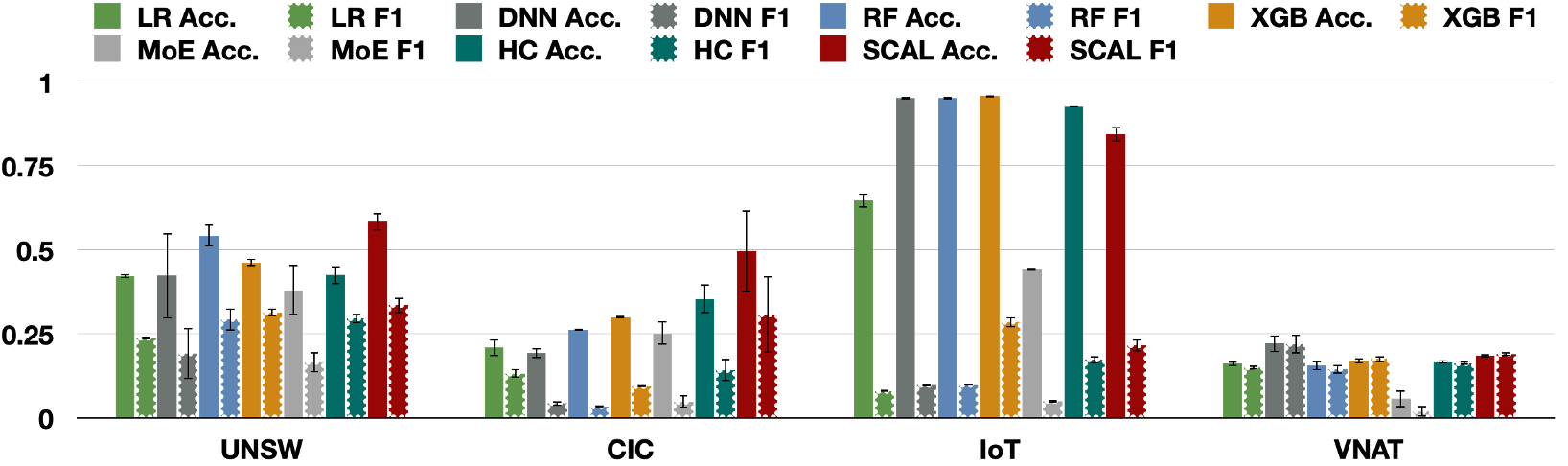}
    \caption{\label{fig:attack_eps12} Accuracy and F1 scores after adversarial attacks with $\epsilon = 1.2$ against  Logistic Regression (LR in green), Deep Neural Network (DNN in grey), Random Forest (RF in blue), XGBoost (XGB in orange), Mixture of Experts (MoE in light grey),  Binary Hierarchical Classification (HC in teal green), and SCAL (SCAL in red) on the four datasets.}
\end{figure*}

We also varied the number of estimators in the global random forest to examine the accuracy-complexity trade-off. As shown in Figure~\ref{fig:comp}, increasing the number of trees in random forest improved F1 scores but grew model complexity by two to five times. In contrast, SCAL achieved higher F1 scores with a fraction of the model complexity, demonstrating that increasing global model complexity cannot match the predictive performance of correlation-aware local models.

\subsection{Robustness} 
We evaluate adversarial robustness by applying the Cube attack~\cite{andriushchenko2019provably}---an efficient black-box $L_{\infty}$ attack that perturbs random feature subsets. The attack was selected for its effectiveness against gradient-free models such as random forests and decision trees~\cite{10.1007s10208-015-9296-2}. Attacks against SCAL were conducted at two levels: first targeting the instance distributor to misdirect instances to the wrong subproblem, and if unsuccessful, targeting the local classifiers directly. This constitutes a grey-box attack requiring partial knowledge of the cascade structure.


Figure~\ref{fig:attack_eps12} shows results under mild attacks ($\epsilon=1.2$). Key observations are:
\begin{itemize}
\itemsep=0pt
\item F1 scores dropped significantly for all classifiers across all datasets.
\item On the IoT dataset, five classifiers retained high accuracy by defaulting to the majority \textit{Benign} class.
\item SCAL was comparable to or more resilient than all other classifiers, despite facing stronger grey-box attacks at both the instance distributor and local classifier levels.
\end{itemize}

Under a stronger attack budget ($\epsilon=2.2$), performance deteriorated sharply for all classifiers, yet SCAL remained the most or equally resilient across all datasets (detailed results omitted due to page limitations). 

In summary, SCAL demonstrated consistently superior adversarial robustness across all attack budgets, even under the more challenging grey-box setting.

\subsection{Explainability}

Decision trees and rules are widely recognized as interpretable models~\cite{rudin2019explaining, osti_10350681}. While local classifiers in SCAL are simple decision trees, their complexity can still exceed human comprehension, often containing hundreds of thousands of nodes. In this section, we explore pruning local classifiers using GENESIM~\cite{vandewiele2016genesim}, a multi-objective optimization algorithm that aggressively reduces model complexity while preserving accuracy.

Figure~\ref{fig:explain} shows model size (log number of nodes) and predictive performance before and after pruning. Results are as follows:

\begin{itemize}
\itemsep=0pt
\item \textbf{UNSW}: Pruning reduced tree size fourfold with no accuracy loss; F1 decreased slightly as the algorithm optimizes for accuracy.
\item \textbf{CIC}: Tree size was reduced by up to 4,500 times to just 177 nodes with no accuracy loss, though F1 dropped due to aggressive pruning of minority classes.
\item \textbf{IoT}: The trained tree was already compact ($\sim$250 nodes), so pruning was not applied.
\item \textbf{VNAT}: Tree size was reduced 293 times to 122 nodes with no loss in accuracy or F1.
\end{itemize}

Pruning merged all local trees into a single compact model of just a few dozen nodes, substantially improving interpretability. Notably, aggressive pruning of global random forest classifiers typically causes significant accuracy loss~\cite{vandewiele2016genesim}, whereas pruned local classifiers in SCAL retain predictive performance.

\begin{figure}[!htb]
    \centering
        \includegraphics[clip,trim=0cm 0cm 0cm 0cm, width=0.49\textwidth]{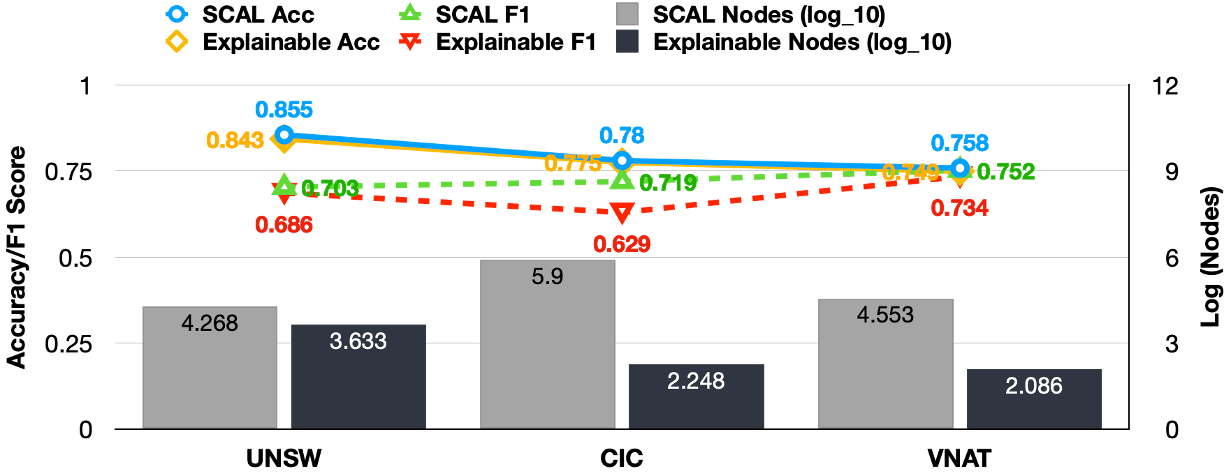}
        \caption{\label{fig:explain} Explainability of global classifier vs local classifier}
\end{figure}


\section{Conclusions and Limitations}
\label{sec:conclude}

We presented SCAL, a divide-and-conquer technique that automatically partitions a complex learning problem into simpler subproblems, and conducted a holistic evaluation across accuracy, cost, robustness, and explainability. Experiments on real-world network intrusion detection datasets confirm that local models combined outperform a single global model on all four dimensions: they are more accurate on minority classes, orders of magnitude smaller, more robust to adversarial attacks, and more interpretable.

\subsubsection*{Limitations}
The current fingerprint generation is designed for network payload data; extending it to standard feature inputs requires further investigation. Robustness was evaluated using a single black-box attack effective against all models; future work should employ model-specific attacks and explore varying attack strengths during retraining and evaluation. Explainability was limited to decision trees; future studies should incorporate model-agnostic methods such as LIME~\cite{10114529396722939778} and SHAP~\cite{10555532952223295230}. Despite these limitations, SCAL demonstrates strong potential for addressing all four machine learning concerns in resource-constrained, adversarial environments.
  


\bibliography{milcom}
\bibliographystyle{milcom_style/IEEEtran}
\end{document}